\documentclass{article}

\usepackage[utf8]{inputenc}
\usepackage{amsmath,amscd,amssymb,amsthm}

\usepackage{verbatim}
\usepackage{thmtools}
\usepackage{thm-restate}
\newtheorem{definition}{Definition}
\newtheorem{lemma}[definition]{Lemma}
\newtheorem{theorem}[definition]{Theorem}
\newtheorem{corollary}[definition]{Corollary}

\usepackage{tikz}
\usetikzlibrary{shapes,arrows}
\usetikzlibrary{positioning}
\usepackage{fullpage}

\tikzstyle{reduction} = [rectangle, draw, fill=blue!20,  align=center,
    text centered, rounded corners, minimum height=4em, minimum width=4em]
\tikzstyle{reductiongreen} = [rectangle, draw=red!50!green!100, fill=red!40!green!40, align=center,
    text centered, rounded corners, minimum height=4em, minimum width=4em]
\tikzstyle{reductionpurple} = [rectangle, draw=purple, fill=purple!40, align=center,
    text centered, rounded corners, minimum height=4em, minimum width=4em]
\tikzstyle{reductionblue} = [rectangle, draw=blue, fill=blue!20, align=center,
    text centered, rounded corners, minimum height=4em, minimum width=4em]
\tikzstyle{application} = [rectangle, draw, fill=red!20, align=center,
    text centered, minimum height=4em, minimum width=4em, rounded corners]
\tikzstyle{application2} = [rectangle, draw, fill=black!20, align=center,
    text centered, minimum height=4em, minimum width=4em, rounded corners]
\tikzstyle{line} = [draw, -latex']
\tikzstyle{setting} = [rectangle, draw, fill=blue!20, align=center,
    text centered, rounded corners, minimum height=4em, minimum width=4em]
\tikzstyle{bound} = [rectangle, draw, fill=red!20, align=center,
    text centered, rounded corners, minimum height=4em, minimum width=4em]
\tikzstyle{thickline} = [draw, -latex', thick]

\usepackage[framemethod=TikZ]{mdframed}
\mdfdefinestyle{MyFrame}{%
    linecolor=black,
    outerlinewidth=2pt,
    roundcorner=20pt,
    innerrightmargin=20pt,
    innerleftmargin=20pt,
    backgroundcolor=white}
\usepackage{natbib}
\setcitestyle{numbers}
\setcitestyle{square}
\usepackage{hyperref}

\usepackage{caption}

\usepackage{appendix}
\usepackage{chngcntr}
\usepackage{etoolbox}
\usepackage{lipsum}

\AtBeginEnvironment{subappendices}{%
\chapter*{Appendix}
\addcontentsline{toc}{chapter}{Appendices}
\counterwithin{figure}{section}
\counterwithin{table}{section}
}

\newcommand{\argmin}{\mathop{\text{argmin}}}

\newcommand{\ol}{\mathcal{A}}
\newcommand{\bol}{\mathcal{A}_{B}}
\newcommand{\onedol}{\mathcal{A_{\text{1D}}}}

\newcommand{\wealth}{\text{Wealth}}
\newcommand{\w}{\mathring{w}}
\renewcommand{\v}{\mathring{v}}
\newcommand{\trunc}{\text{trunc}}
\newcommand{\G}{\mathfrak{G}}

\newcommand{\algname}{\textsc{Leashed}}

\newcommand{\field}[1]{\mathbb{#1}}

\newcommand{\R}{\field{R}}

\newcommand{\E}{\field{E}}

\usepackage{graphicx}

\usepackage{algorithm}
\usepackage{algpseudocode}

\title{Artificial Constraints and Lipschitz Hints for Unconstrained Online Learning}
\author{Ashok Cutkosky\\
Google\\
\texttt{ashok@cutkosky.com}
}
\date{}

\begin{document}

\maketitle

\begin{abstract}%
We provide algorithms that guarantee regret $R_T(\w)\le \tilde O(G\|\w\|^3 +  G(\|\w\|+1)\sqrt{T})$ or $R_T(\w)\le \tilde O(G\|\w\|^3T^{1/3} + GT^{1/3}+ G\|\w\|\sqrt{T})$ for online convex optimization with $G$-Lipschitz losses for any comparison point $\w$ without prior knowledge of either $G$ or $\|\w\|$. Previous algorithms dispense with the $O(\|\w\|^3)$ term at the expense of knowledge of one or both of these parameters, while a lower bound shows that some additional penalty term over $G\|\w\|\sqrt{T}$ is necessary. Previous penalties were exponential while our bounds are polynomial in all quantities. Further, given a known bound $\|\w\|\le D$, our same techniques allow us to design algorithms that adapt optimally to the unknown value of $\|\w\|$ without requiring knowledge of $G$.
\end{abstract}


\section{Unconstrained Online Convex Optimization}

Online convex optimization (OCO) is a popular theoretical framework for designing algorithms that operate on streams of input data \citep{shalev2011online, zinkevich2003online}. Such problems abound in today's world of extremely large datasets, and so many empirically successful algorithms are analyzed in the OCO framework (e.g. \citep{duchi10adagrad, ross2013normalized, mcmahan2013ad}). In detail, OCO is a game consisting of $T$ rounds. In each round, the learning algorithm first outputs a point $w_t$ in some Hilbert space $W$\footnote{Our results also apply in Banach spaces, but we focus our exposition on Hilbert spaces for simplicity}, and then the environment outputs a convex loss function $\ell_t:W\to \R$. The learner's goal is to minimize the \emph{regret}, which is the total loss suffered by the learner in comparison to the loss suffered at some benchmark point $\w\in W$:
\[
R_T(\w)=\sum_{t=1}^T \ell_t(w_t)-\ell_t(\w)
\]

Online learning algorithms can be naturally applied to stochastic optimization problems, in which each $\ell_t$ is an i.i.d. random variable with $\E[\ell_t] = \mathcal{L}$ for some fixed loss function $\mathcal{L}$. In this case, the online-to-batch conversion argument \citep{littlestone2014line} tells us that $\E[\mathcal{L}(\overline{w})-\mathcal{L}(\w)]\le \tfrac{\E[R_T(\w)]}{T}$, where $\overline{w}=\tfrac{\sum_{t=1}^T w_t}{t}$. Thus, we refer to $\tfrac{R_T(\w)}{T}$ as the \emph{convergence} rate. We wish to achieve a convergence rate such that $\lim_{T\to \infty} \tfrac{R_T(\w)}{T} = 0$, known as sublinear regret.

We can reduce OCO to online \emph{linear} optimization (OLO) in which each loss $\ell_t$ must be linear through the use of subgradients: if $g_t\in \partial \ell_t(w_t)$, we have
\begin{align}
R_T(\w) \le \sum_{t=1}^T \langle g_t, w_t-\w\rangle\label{eqn:linearregret}
\end{align}
Thus by supplying the linear losses $w\mapsto \langle g_t, w\rangle$ to an OLO algorithm, we obtain an OCO algorithm. Because it is often easy to compute gradients, many of the popular OCO algorithms are in fact OLO algorithms. We will also take this road, and consider exclusively the case that losses are linear functions specified by the subgradients $g_t$ so that (\ref{eqn:linearregret}) is an equality.

Our goal is to design algorithms which take no data-dependent parameters as input and yet nevertheless guarantee regret matching the minimax optimal regret for online linear optimization algorithms, sometimes called ``parameter-free'' algorithms \citep{abernethy2008optimal, mcmahan2012no, orabona2013dimension, foster2015adaptive, foster2017parameter}. Specifically, we want to obtain:
\begin{align}
R_T(\w)&=O\left(\|\w\|\sqrt{\sum_{t=1}^T \|g_t\|_\star^2\log\left(GT+1\right)}\right)\nonumber\\
&= O\left(\|\w\|G\sqrt{T\log\left(GT+1\right)}\right)\label{eqn:goal}
\end{align}
for all $\w$ simultaneously, where $G=\max \|g_t\|_\star$, $\|\cdot\|$ is some norm and $\|\cdot\|_\star$ is the dual norm. Classical gradient-descent algorithms require learning rates that are tuned to the values of $\|\w\|$ and $\|g_t\|_\star$, while parameter-free algorithms automatically adapt to these unknown parameters, and so can largely dispense with tuning.

\subsection{Previous Progress and Lower Bounds}

Previous approaches to designing parameter-free algorithms typically relax the problem by assuming either a bound $D$ on $\|\w\|$, or a bound $\G$ on $G=\max_t \|g_t\|_\star$. In the case of known value for $D$, classical approaches based on Follow-the-Regularized-Leader or Mirror Descent with strongly-convex potentials can obtain regret $O\left(D\sqrt{\sum_{t=1}^T \|g_t\|^2}\right)$ \citep{duchi10adagrad,hazan2008adaptive, orabona2016scale}. In the alternative when a bound $G\le \G$ is known, the problem seems somewhat harder. Most prior work in this setting instead obtains the slightly-less-good regret bound $R_T(\w)=\tilde O\left(\|\w\|\sqrt{\G\sum_{t=1}^T \|g_t\|_\star}\right)$ \citep{orabona2014simultaneous,orabona2017backprop,cutkosky2017online}, and the better rate of $\tilde O\left(\|\w\|\sqrt{\G^2+\sum_{t=1}^T \|g_t\|_\star^2}\right)$ has only recently been developed \citep{cutkosky2018black}.

When neither $\|\w\|$ nor $\G$ is known, prior lower bounds show that it is actually impossible to obtain the desired result \citep{cutkosky2016online,cutkosky2017online}. Instead, there is a frontier of lower-bounds trading off between the ``ideal'' $\tilde O(\|\w\|G\sqrt{T})$ term and an \emph{exponential} penalty of up to $O(\exp(T)/T)$, depending on how fast the gradients grow.

We will improve upon this background in two ways. First, observe that there is a curious asymmetry between the results for known $D$ and known $\G$. On the one hand, in the known $D$ case the algorithms do not adapt very well to the true value of $\|\w\|$, instead obtaining a regret bound proportional to $DG\sqrt{T}$. On the other hand, in the known $\G$ case, so long as $\G$ is not greater than $G\sqrt{T}$, the resulting bound depends on $\|\w\|G\sqrt{T}$. So in a sense, the known $\G$ algorithms are more robust than than the known $D$ ones. We close this gap by providing an algorithm that takes an upper bound $D$ but no bound $\G$ and maintains regret $\tilde O(\|\w\|G\sqrt{T}+DG)$. The principal technique used in this analysis will then allow us to improve in a second way: when neither $\G$ nor a bound $D$ is known, we design an algorithm that still guarantees sublinear regret and whose only $T$-dependent regret terms are linear in $\|\w\|$. 

The lower bound makes it seem that there is no hope for this second goal - but it is not so! The key observation is that the exponential lower bound applies only to algorithms which insist on a regret bound whose dependence on $\|\w\|$ is $\tilde O(\|\w\|)$. However, in practice we suspect that it is more important to maintain a $T$-dependence of $\sqrt{T}$. As a result, we will replace the aesthetically unappealing $\exp(T)/T$ term with a more palatable penalty of $G\|\w\|^3$. This new regret bound is incomparable to the previous lower bound, in the sense that neither function is dominated by the other: for very large $\|\w\|$, $\|\w\|^3 \ge \exp(T)/T$. However, we note that our new penalty is \emph{independent of $T$}, and so as $T$ becomes large compared to $\|\w\|$, our asymptotic convergence rate is guaranteed to match the optimal rate. To our knowledge, ours is the first guarantee of this kind in online learning.

\section{Overview of Techniques and Results}

For any user-specified $p$ and $k$, We will construct an online linear optimization algorithm, $\algname$, that guarantees the regret bound:
\begin{align}
\nonumber R_T(\w)&\le \tilde O\left[ \|\w\|\sqrt{\sum_{t=1}^T\|g_t\|_\star^2} +G_Tk\left(\max_{t}\frac{\sum_{t'=1}^t \|g_{t'}\|_\star}{G_t}\right)^{p}\right.\\\
&\quad\quad\left.+ G_T\min_{q\in[0,1]}\frac{\|\w\|^{1+\frac{1-q}{p}}} {k^{\frac{1-q}{p}}}\left(\frac{\sum_{t=1}^T\|g_t\|_\star}{G_T}\right)^q\right]\label{eqn:bound1}
\end{align}
where $G_t=\max_{t'\le t} \|g_{t'}\|_\star$ and $\|\w\|$ are unknown to the algorithm ahead of time and $\tilde O$ hides a logarithmic factor. Using this result, we can apply $p=1/2$ and $q=0$ to obtain the bound:
\begin{align*}
R_T(\w) &\le \tilde O\left(\|\w\|\sqrt{\sum_{t=1}^T \|g_t\|_\star^2}+G_Tk\sqrt{\max_{t}\frac{\sum_{t'=1}^{t} \|g_{t'}\|_\star}{G_t}} + G_T\frac{\|\w\|^3}{k^2}\right)\\
&=\tilde O\left((\|\w\|+k)G_T\sqrt{T} + G_T\frac{\|\w\|^3}{k^2}\right)
\end{align*}
Another interesting setting is $p=1/3$, $q=1/3$, $k=1$, which yields $\tilde O(G\|\w\|^3T^{1/3}+GT^{1/3} +\|\w\|G\sqrt{T})$ and makes the dominant $T$-dependent term $\tilde O(\|\w\|G\sqrt{T})$ without any $\|\w\|+k$ terms. 

Finally, observe that setting $p=1$, $q=0$ in (\ref{eqn:bound1}) yields an expression (up to logs) of the form $G\|\w\|^2/k + GkT + \|\w\|G\sqrt{T}$, which is reminiscent of the results in FTRL or Mirror Descent analysis with strongly-convex regularizers, in which one needs to tune the ``learning rate'' $k$ to be $O\left(\frac{1}{\sqrt{T}}\right)$ in order to obtain $O(\sqrt{T})$ regret.

The term $\max_{t}\frac{\sum_{t'=1}^{t} \|g_{t'}\|_\star}{G_t}$ may seem a little awkward. It appears for technical reasons and it is very easy to replace it with $C\sum_{t=1}^{T} \|g_{t}\|_\star$ for any user-specified constant $C$. We choose not to do this in our presentation in order to allow our regret bound to be nearly scale-free, in the sense that scaling all the gradients by any constant scales the bound by the same constant.



\subsection{Outline and Proof Steps}
First, we observe that it suffices to achieve our desired bounds in the one-dimensional case $W=\R$, as it is easy to convert any one-dimensional algorithm to a dimension-free algorithm via a recent reduction argument \citep{cutkosky2018black} (see Section \ref{sec:dimfree} for details). Our one-dimensional algorithm is then constructed via three steps:
\begin{enumerate}
    \item First, we develop an online linear optimization algorithm for a modified setup in which the algorithm tis given access to a sequence of ``hints'' $h_1\le\dots\le h_T$ such that $|g_t|\le h_t$ and $h_t$ is revealed before $g_t$. The algorithm utilizes the hints to avoid suffering any penalty for a priori unknown bounds on $G$. This step is the most technical step, although it is essentially just a careful verification that prior analysis is completely unchanged when incorporating these ``just in time'' bounds on $|g_t|$ (Section \ref{sec:step1}).
    \item Although these hints aren't actually available, we approximate them by $h_t=\max_{i<t} |g_i|$. We then analyze the error from this approximation, and show that it results in adding a penalty of $G\max_t |w_t|$ to the regret (Section \ref{sec:step2}).
    \item From the previous step, it seems that we should try to control $\max_t |w_t|$. We do this by enforcing an ``artificial constraint'': we use the constraint set reduction in \citep{cutkosky2018black} to ensure $|w_t|\le \sqrt{T}$ for all $t$. This results in good regret for all $|\w|\le \sqrt{T}$, but does not control regret for $|\w|>\sqrt{T}$. To address $|\w|> \sqrt{T}$, we then observe that $R_T(\w) \le  R_T(0) + |\w|GT\le R_T(0)  + G|\w|^3$ and use the fact that $R_T(0)$ is constant (because $|0|\le \sqrt{T}$) to conclude the desired results (Section \ref{sec:step3}).
\end{enumerate}

We point out that the second step of our proof actually gives a rather general way to convert algorithms that require bounded losses into ones that do not require bounded losses, so long as the domain $W$ is itself bounded. Thus, it is our hope that it may have broader applicability.

\section{Step 1: Unconstrained Optimization With Hints}\label{sec:step1}

Our overall approach can be viewed as a sequence of elaborate ``tricks'' designed to convert an algorithm that requires a Lipschitz bound $G$ into one that does not. Our first step in this section is to consider a slightly easier OLO game in which the algorithm is given access to ``hints'' $h_t$ that bound the next loss $g_t$. Formally, each round of the OLO game now consists of:
\begin{enumerate}
\item Learner receives hint $h_t\in \R$, with $h_t\ge h_{t-1}$.
\item Learner plays $w_t$.
\item Learner receives loss $g_t$ with $|g_t|\le h_t$.
\end{enumerate}
We will look for regret bounds that depend both on $\w$ as well as the hints, taking the form:
\begin{align*}
\sum_{t=1}^T g_t(w_t-\w)&\le R_T(\w,h_T)
\end{align*}

To accomplish this, we adapt the algorithm of \citep{cutkosky2018black} to this setting and obtain an algorithm that guarantees regret
\begin{align*}
R_T(\w)&= \tilde O\left[|\w|\max\left(h_T,\ \sqrt{\sum_{t=1}^T g_t^2}\right)\right]
\end{align*}
Previous algorithms \citep{orabona2014simultaneous, cutkosky2017online} can also be adapted to this setting, but the dependencies on $h_T$ are worse.

The algorithm (Algorithm \ref{alg:ons_1d}) operates in the coin-betting framework \citep{orabona2016coin}. A coin-betting algorithm achieves low regret by maintaining high \emph{wealth}, defined by $\wealth_T=\epsilon - \sum_{t=1}^T g_tw_t$ for some user-specified $\epsilon>0$. The wealth is increased by playing $w_t=v_t\wealth_t$ for some $v_t\in[-1,1]$, which corresponds to ``betting'' a fraction $v_t$ of $\wealth_t$ on the value of the ``coin'' $g_t$, because $\wealth_{t+1} = \wealth_t - g_tv_t\wealth_t$. Thus the problem of choosing $w_t$ reduces to the problem of choosing $v_t$. We solve this problem in the same way as \citep{cutkosky2018black} by recasting choosing $v_t$ as an online exp-concave optimization  problem in 1 dimension, and then use the Online Newton Step (ONS) algorithm \citep{hazan2007logarithmic} to optimize $v_t$.

This strategy can become nonsensical if the value of $\wealth_t$ ever becomes non-positive. To avoid this, we wish to guarantee $|v_t|< 1/|g_t|$ for all $t$, which implies $\wealth_{t+1}=(1-v_tg_t)\wealth_t> 0$. Thus, we use our ``hint'' $h_t$ and restrict the value $v_t$ to the range $[-1/2h_t, 1/2h_t]$ to achieve the desired outcome.

\begin{algorithm}[h]
\caption{Coin-Betting through ONS With Hints}
\label{alg:ons_1d}
\begin{algorithmic}[1]
{
    \Require{Initial wealth $\epsilon>0$, parameter $\alpha>0$}
    \State{{\bfseries Initialize: } $\wealth_0=\epsilon$, initial betting fraction $v_1 = 0$, initial hint $h_1$}
    \For{$t=1$ {\bfseries to} $T$}
    \State{Bet $w_t = v_t \, \wealth_{t-1} $}
    \State{Receive $g_t\le h_t$}
    \State{Receive $h_{t+1}\ge h_t$}
    \State{Update $\wealth_{t} = \wealth_{t-1} - \tilde g_t w_t $}
    \State{compute new betting fraction $v_{t+1}\in[-1/2G_{t+1},1/2G_{t+1}]$ via ONS update on losses $-\ln(1-\tilde g_tv)$}
    \State{Set $z_t = \frac{d}{dv_t}\left(-\ln(1-\tilde g_t v_t)\right)=\frac{\tilde g_t}{1-g_t v_t}$}
    \State{Set $A_{t} = 4\alpha + \sum_{i=1}^t z_i^2$}
    \State{$v_{t+1} = \max\left(\min\left(v_{t} - \frac{2}{2-\ln(3)} \frac{z_t}{A_{t}}, \frac{1}{2h_{t+1}}\right),-\frac{1}{2h_{t+1}}\right)$}
    \EndFor
}
\end{algorithmic}
\end{algorithm}

The analysis of this algorithm is nearly identical to that in \citep{cutkosky2018black}, although we reproduce the main steps for completeness. The major deviation is that we are performing ONS updates on shrinking domains $[-1/2h_t,1/2h_t]$, which we analyze carefully in appendix.

\begin{restatable}{theorem}{onshints}\label{thm:onshints}
The regret of Algorithm \ref{alg:ons_1d} is bounded by:
\begin{align*}
R_T(\w,h_T) &\le \epsilon + |\w| 
\max\left(8h_T \left(\ln \frac{16|\w| h_T\exp(\alpha/4h_T^2) \left(1+\frac{\sum_{t=1}^T g_t^2}{\alpha}\right)^{4.5}}{\epsilon}-1\right),\right.\\
&\quad\quad\left.
 2\sqrt{\sum_{t=1}^T g_t^2 \ln\left(\frac{4 \left(\sum_{t=1}^T g_t^2\right)^{10}\exp(\alpha/2h_T^2) \w^2}{\epsilon^2}+1\right)}\right)
\end{align*}
\end{restatable}



\section{Step 2: Without Hints, Regret is Small if $\|w_t\|$ is Small}\label{sec:step2}

In this section we remove the need for externally supplied hints used in Section \ref{sec:step1}. The technique presented here may apply more generally than our present focus, so we state the technique in terms of general norms, rather than restricting to a one-dimensional problem as we do in the other sections. Specifically, we show how to convert any algorithm that uses hints $h_t$ and obtains regret $R_T(\w,h_T)$ into one that does not receive hints and obtains regret $R_T(\w,G) + G\max_t \|w_t\|$ (where again $G=\max_t \|g_t\|_\star$). The procedure is very simple: we run the algorithm using (incorrect!) hints $h_t = \max_{i< t} \|g_i\|_\star$. Whenever we observe $\|g_t\|> h_t$ (i.e. when the hint is wrong), we ``lie'' to the algorithm: we replace $g_t$ with a ``truncated gradient'' $g^{\trunc}_t = h_t\frac{g_t}{\|g_t\|}$ to make the hint correct, and then bound the error produced by this truncation. Pseudocode is provided in Algorithm \ref{alg:unknowngt}.

\begin{algorithm}[h]
\caption{Algorithm Without Hints}
\label{alg:unknowngt}
\begin{algorithmic}[1]
{
    \Require{Algorithm $\ol$ that takes hints $h_t$, initial value $\mathfrak{g}$}
    \State{{\bfseries Initialize: } initialize $h_1=\mathfrak{g}$}
    \For{$t=1$ {\bfseries to} $T$}
    \State{Get $w_t$ from $\ol$, play $w_t$}
    \State{Receive $g_t$}
    \If{$|g_t|\ge h_t$}
    \State{$g^{\trunc}_t \gets h_t\frac{g_t}{\|g_t\|_\star}$}
    \Else
    \State{$g^{\trunc}_t\gets g_t$}
    \EndIf
    \State{$h_{t+1} = \max(h_t, \|g_t\|_\star)$}
    \State{send $g^{\trunc}_t$ and $h_{t+1}$ to $\ol$}
    \EndFor
}
\end{algorithmic}
\end{algorithm}

\begin{theorem}\label{thm:unknowngt}
Suppose $\ol$ obtains $R_T(\w, h_T)$ given hints $h_1\le\dots\le h_T$. Then Algorithm \ref{alg:unknowngt} obtains
\[
R_T(\w) \le R_T(\w, \max(\mathfrak{g}, G)) + G\max_t \|w_t\| + G\|\w\|
\]
where $G=  \max_t \|g_t\|_\star$.
\end{theorem}
\begin{proof}
First, we observe that the gradients $g^{\trunc}_t$ provided to $\ol$ do indeed respect the hints, $\|g^{\trunc}_t\|_\star \le h_t$. Thus we have:
\begin{align*}
\sum_{t=1}^T g^{\trunc}_t \cdot w_t -g^{\trunc}_t\cdot \w\le R_T(\w, h_T)
\end{align*}
Moving on to the true regret, we define $G_t=\max_{i\le t} \|g_i\|_\star$ for convenience, and then compute:
\begin{align*}
\sum_{t=1}^T g_t\cdot w_t-g_t\cdot \w&=\sum_{t=1}^T g^{\trunc}_t\cdot (w_t-\w) + (g_t-g^{\trunc}_t)\cdot (w_t-\w)\\
&\le R_T(\w, h_T) + (\|\w\| + \max_t\|w_t\|)\sum_{t=1}^T \|g_t-g^{\trunc}_t\|_\star\\
&\le R_T(\w, h_T) + (\|\w\| + \max_t\|w_t\|)\sum_{t| h_t<G_t}G_t-h_t\\
&\le R_T(\w, h_T) + (\|\w\| + \max_t\|w_t\|)\sum_{t| h_t<G_t}G_t-G_{t-1}\\
&\le R_T(\w, h_T) + (\|\w\| + \max_t\|w_t\|)G
\end{align*}
where we have observed $h_t\ge G_{t-1}$ in the second-to-last line. Now we see that $h_T=\max(\mathfrak{g}, G)$ to complete the proof.
\end{proof}

Combining this result with our Algorithm \ref{alg:ons_1d}, we obtain a regret bound of
\begin{align*}
R_T(\w) &\le \tilde O\left[\|\w\|\sqrt{\sum_{t=1}^T \|g_t\|_\star^2} + \max(\mathfrak{g}, G)\|\w\| + G\max_t \|w_t\| + \epsilon\right]
\end{align*}
where we have temporarily suppressed all logarithmic factors for ease of exposition. Intuitively, we obtain nearly the same regret guarantee as before, but suffer an additional penalty that is small so long as the $w_t$s do not grow too much.

This reduction allows us to address the asymmetry between the prior algorithms with known $D$ versus known $\G$. Specifically, when we operate within a setting with bounded diameter $D$, then $\max_T\|w_t\|\le D$. We can construct an algorithm with domain $W$ that takes hints and obtains regret that adapts to $\|\w\|$ by applying the one-dimensional-to-dimension-free reductions and unconstrained-to-constrained reductions of \citep{cutkosky2018black} to Algorithm \ref{alg:ons_1d}. Then by appling the reduction of this section and leveraging $\max_T \|\w_t\|\le D$, we obtain the regret bound:
\begin{align*}
R_T(\w)&\le \tilde O\left(\|\w\|\sqrt{\sum_{t=1}^T \|g_t\|_\star^2} + DG\right)
\end{align*}
Further, by applying these reductions to Algorithm 6 of \citep{cutkosky2018black} we can obtain:
\begin{align*}
R_T(\w) &\le \tilde O\left(\sqrt{\sum_{t=1}^T \|w_t-\w\|^2\|g_t\|_\star^2}+DG\right)
\end{align*}
The latter regret bound is of interest as it implies logarithmic regret on strongly-convex losses without requiring any knowledge of the strong-convexity parameter or any Lipschitz bounds. We expect that other algorithms involving bounded domains and Lipschitz bounds can also take advantage of this technique to remove the Lipschitz bound requirement.

\section{Step 3: Artificial Constraints}\label{sec:step3}



Returning to a one-dimensional problem, in this last step, we leverage the result of the previous section by preventing our algorithm from choosing $w_t$s with overly-large magnitudes. To gain some intuition for our strategy, suppose we can constrain the algorithm in the previous section to the set $\left[-\sqrt{\sum_{t=1}^T |g_t|},\sqrt{\sum_{t=1}^T |g_t|}\right]$, while still maintaining the same regret bound for any $\w$ in this interval\footnote{Note that we don't know this interval a priori - this is just a thought-experiment to gain intuition.}. 
This enforces $\max_t \|w_t\|\le \sqrt{\sum_{t=1}^T \|g_t\|}$, so that intuitively we have a bound of
\begin{align*}
R_T(\w)&\le \tilde O\left[\epsilon+G\sqrt{\sum_{t=1}^T |g_t|} +  |\w|\sqrt{\sum_{t=1}^T g_t^2 }+\epsilon\right]~.
\end{align*}
for any $\w$ with $|\w|\le \sqrt{\sum_{t=1}^T |g_t|}$. Thus it remains to address $\w$ outside the constraining interval.

For $\w$ outside the interval, we have
\begin{align*}
R_T(\w)&=\sum_{t=1}^T g_tw_t-g_t\w\\
&\le R_T(0) + |\w|\sum_{t=1}^T |g_t|\\
&\le R_T(0) + |\w|^3
\end{align*}
Where in the last step we used $|\w|\ge \sqrt{\sum_{t=1}^T |g_t|}$. Combining the two guarantees, with our result from the previous section for which $R_T(0)\le \epsilon$, we have for all $\w$:
\begin{align*}
R_T(\w)&\le \tilde O\left[|\w|^3+G\sqrt{\sum_{t=1}^T |g_t|} + G|\w|+\epsilon + \max(\mathfrak{g}, G)\|\w\| +  |\w|\sqrt{\sum_{t=1}^T g_t^2 }\right]~.
\end{align*}

There are two issues with this intuition that need to be addressed. First, we need to show how to restrict to the desired interval without affecting the regret bound for $\w$ inside the interval, and second we need to deal with the fact that we do not know the value of $\sum_{t=1}^T |g_t|$ apriori. We address the first issue by appealing to the constraint-set reduction of \citep{cutkosky2018black}, which provides exactly the desired mechanism (full details are reproduced in our proof of Theorem \ref{thm:final}). We address the second issue in greater generality by considering a $t$-varying constraint-set $\left [-k\left(\sum_{i=1}^{t-1} |g_i|\right)^p,k\left(\sum_{i=1}^{t-1}|g_i|\right)^p\right]$ for user-specified $p$ and $k$.

We present the pseudo-code for the final algorithm in Algorithm \ref{alg:fullyparameterfree}, \algname\ below.





\subsection{The Final Algorithm}\label{sec:final}

\begin{algorithm}[h]
\caption{\algname}
\label{alg:fullyparameterfree}
\begin{algorithmic}[1]
{
    \Require{Algorithm $\ol$ that takes hints, parameters $k,p,\mathfrak{g},\epsilon$}
    \State{{\bfseries Initialize: } initialize $h_1=\mathfrak{g}$, $G_0=0$, $B_1=0$}
    \State{Send initial hint $h_1$ to $\ol$}
    \For{$t=1$ {\bfseries to} $T$}
    \State{Get $w_t$ from $\ol$}
    \If{$|w_t|\ge B_t$}
    \State{//project to artificial constraint set $[-B_t, B_t]$}
    \State{$\tilde w_t\gets B_t\frac{w_t}{|w_t|}$}
    \Else
    \State{$\tilde w_t\gets w_t$}
    \EndIf
    \State{Play $\tilde w_t$, receive $g_t$}
    \State $G_t \gets \max(G_{t-1}, |g_t|)$.
    \State $h_{t+1}\gets \max(h_t, |g_t|)$.
    \State{//update artificial constraint}
    \State $B_{t+1}\gets k\left(\sum_{i=1}^{t} |g_t|/G_t\right)^p$
    \State{//deal with increasing gradient sizes}
    \If{$|g_t|\ge h_t$}
    \State{//replace $g_t$ with truncated version}
    \State{$g^{\trunc}_t\gets h_t\frac{g_t}{|g_t|}$}
    \Else
    \State{//no need to modify $g_t$}
    \State{$g^{\trunc}_t \gets g_t$}
    \EndIf
    \State{//modify gradient to respect artificial constraint}
    \State{Set $\tilde \ell_t(w) = \frac{1}{2}\left(g^{\trunc}_tw + |g^{\trunc}_t|\max(0,|w|-B_t\right))$}
    \State{Compute $\tilde g_t\in \partial \tilde\ell_t(w_t)$}
    \State{Send $\tilde g_t$ and $h_{t+1}$ to $\ol$}
    \EndFor
}
\end{algorithmic}
\end{algorithm}

\begin{theorem}\label{thm:final}
Suppose $\ol$ guarantees regret $R_T^{\ol}(\w,h_T)$ given gradients $\tilde g_1,\dots,\tilde g_T$ and hints $h_1\le\dots\le h_T$ such that $|\tilde g_t|\le h_t$. Then \algname\ obtains regret
\begin{align*}
R_T(\w)&\le  2R_T^{\ol}(\w, \max(\mathfrak{g},G)) + Gk\left[\max_{t\le T}\left(\sum_{t=i}^{t} |g_i|/G_t\right)^p\right] + 2G|\w| \\
&\quad\quad+ \min_{q\in [0,1]}G\left[\frac{|\w|^{1+\frac{1-q}{p}}}{k^{\frac{1-q}{p}}}\left(\sum_{t=1}^{T} |g_t|/G\right)^{q}\right]
\end{align*}
where $G=\max_{t\le T} |g_t|$.
\end{theorem}


\begin{proof}
The additional components of this proof over that of Theorem \ref{thm:unknowngt} are inspired by the proof of the constraint-set reduction in \citep{cutkosky2018black} (Theorem 3), with some modification to deal with the time-varying constraints. 

First, we mirror the argument of Theorem \ref{thm:unknowngt}:

\begin{align*}
\sum_{t=1}^T g_t(\tilde w_t-\w)&\le \sum_{t=1}^T g^{\trunc}_t(\tilde w_t - \w) + (g_t - g^{\trunc}_t)(\tilde w_t-\w)\\
&\le \sum_{t=1}^T g^{\trunc}_t(\tilde w_t - \w) + \sum_{t=1}^T |g_t - g^{\trunc}_t|(\max_{t} |\tilde w_t| + |\w|)\\
&\le \sum_{t=1}^T g^{\trunc}_t(\tilde w_t - \w) + G\left[\max_t B_t + |\w|\right]
\end{align*}
Now we deal with the first term.  Let $\tilde{\w}_t$ be the projection of $\w$ to $\left[-B_t, B_t \right]$. Then:
\begin{align*}
\sum_{t=1}^T g^{\trunc}_t(\tilde w_t - \w)&\le \sum_{t=1}^T g^{\trunc}_t w_t + |g^{\trunc}_t||\tilde w_t-w_t|-(g^{\trunc}_t\w + |g^{\trunc}_t||\w- \tilde{\w}_t|) + |g^{\trunc}_t||\w- \tilde{\w}_t|\\
&=2\sum_{t=1}^T \tilde \ell_t(w_t) -\tilde\ell_t(\w) + \sum_{t=1}^T|g^{\trunc}_t||\w- \tilde{\w}_t|\\
&\le 2\sum_{t=1}^T \tilde g_t(w_t-\w) + \sum_{t=1}^T|g^{\trunc}_t||\w- \tilde{\w}_t|
\end{align*}
We will analyze these two sums separately. First, observe that $\tilde \ell_t$ is $|g^{\trunc}_t|$-Lipschitz, so that $|\tilde g_t|\le |g^{\trunc}_t|\le |g_t|\le h_t$ for all $t$. Therefore we have
\begin{align*}
2\sum_{t=1}^T \tilde g_t(w_t-\w)&\le 2R_T^{\ol}(\w,h_T)=2R_T^{\ol}(\w, \max(\mathfrak{g},G))
\end{align*}
Where we have observed that $w_t$ is generated by running $\ol$ on gradients $\tilde g_t$ (which satisfy $|\tilde g_t|\le |g_t|$). For the second sum, we have $|g^{\trunc}_t|\le |g_t|$ and $|\w-\tilde{\w}_t|\le |\w|$ so that
\begin{align*}
\sum_{t=1}^T|g^{\trunc}_t||\w- \tilde{\w}_t|&\le \sum_{t|\w\ne \tilde{\w}_t} |g_t||\w|\\
&\le \sum_{t||\w|\ge k\left(\sum_{i=1}^{t-1} |g_i|/G_{t-1}\right)^p} |g_t||\w|
\end{align*}
Let $\mathcal{T}$ be the largest value in $\{1,\dots,T\}$ such that $|\w|\ge B_{\mathcal{T}}= k\left(\sum_{i=1}^{\mathcal{T}-1} |g_t|/G_{\mathcal{T}-1}\right)^p$. Then 
\begin{align*}
\sum_{t=1}^T|g^{\trunc}_t||\w- \tilde{\w}_t|&\le G_T|\w|+ \sum_{t=1}^{\mathcal{T}-1} |g_t||\w|\\
&\le G_T|\w|+\min_{q\in [0,1]}\left[\left(\sum_{t=1}^{\mathcal{T}-1} |g_t|\right)^{q}\left(\sum_{t=1}^{\mathcal{T}-1} |g_t|\right)^{1-q}|\w|\right]\\
&\le G_T|\w|+\min_{q\in [0,1]}\left[\frac{|\w|^{1+\frac{1-q}{p}}}{k^{\frac{1-q}{p}}}G_{\mathcal{T}-1}^{1-q}\left(\sum_{t=1}^{\mathcal{T}-1} |g_t|\right)^{q}\right]\\
&\le G_T|\w|+\min_{q\in [0,1]}\left[\frac{|\w|^{1+\frac{1-q}{p}}}{k^{\frac{1-q}{p}}}G_T\left(\sum_{t=1}^{T} |g_t|/G_T\right)^{q}\right]\\
\end{align*}
where in the second step we used $|\w|\ge k\left(\sum_{i=1}^{\mathcal{T}-1} |g_t|/G_{\mathcal{T}-1}\right)^p$.

Putting all this together, we have
\begin{align*}
R_T(\w)&\le 2R_T^{\ol}(\w, \max(\mathfrak{g},G_T))\\
&\quad\quad + G_T\left[k\min_{t\le T}\left(\sum_{i=1}^{t} |g_i|/G_t\right)^p + 2|\w|\right] + \min_{q\in [0,1]}\left[G_T\frac{|\w|^{1+\frac{1-q}{p}}}{k^{\frac{1-q}{p}}}\left(\sum_{t=1}^{T} |g_t|/G_T\right)^{q}\right]
\end{align*}
\end{proof}

If we combine this reduction with our result from Section \ref{sec:step1}, we obtain the following:
\begin{corollary}\label{thm:detailedregret}
Applying the reduction of Algorithm \ref{alg:fullyparameterfree} to Algorithm \ref{alg:ons_1d}, we guarantee regret:
\begin{align*}
R_T(\w)&\le 2\epsilon + 2|\w| 
\max\left[8h_T \ln \left(\frac{16|\w| h_T\exp(\alpha/4h_T^2) \left(1+\frac{\sum_{t=1}^T g_t^2}{\alpha}\right)^{4.5}}{\epsilon}\right)-h_T,\right.\\
&\quad\quad\left.
 2\sqrt{\sum_{t=1}^T g_t^2 \ln\left(\frac{4 \left(\sum_{t=1}^T g_t^2\right)^{10}\exp(\alpha/4h_T^2) \w^2}{\epsilon^2}+1\right)}\right]\\
&\quad\quad + G\left[k\max_{t\le T}\left(\sum_{i=1}^t |g_i|/G_t\right)^p + 2|\w|\right] + \min_{q\in [0,1]}\left[G \frac{|\w|^{1+\frac{1-q}{p}}}{k^{\frac{1-q}{p}}}\left(\sum_{t=1}^{T} |g_t|/G\right)^{q}\right]
\end{align*}
where $G+\max_t |g_t|$ and $h_T=\max(\mathfrak{g}, G)$.
\end{corollary}

\section{Discussion of Parameters}

Although our algorithm does not need to know the data-dependent parameters $\|\w\|$ and $G$, we nevertheless retain dependence on several user-specified parameters which we discuss in this section. In brief, the parameters are:
\begin{enumerate}
    \item $\epsilon$: The regret at the origin.
    \item $\mathfrak{g}$: Initial hint value, ideally this should be set to an under-estimate of $G$.
    \item $\alpha$: Initial regularizer for ONS.
    \item $k$ and $p$: These control how fast the values of $w_t$ are allowed to grow.
    \item $q$: This exists only for analysis purposes and controls the tradeoff between higher-order dependence on $|\w|$ and lower-order dependence on $T$.
\end{enumerate}

Of these parameters, we observe that $\epsilon$, $\mathfrak{g}$ and $\alpha$ appear only in logarithmic or sub-asymptotic terms. As a result, our algorithm is robust to these parameters. It remains to investigate $k,p$ and $q$, which we do by considering a few settings of interest already highlighted in the introduction.

\begin{enumerate}
\item With the setting $p=1/2$, $q=0$, our regret bound takes the form:
\begin{align*}
R_T(\w)&\le \tilde O\left((|\w|+k)G\sqrt{T} + G|\w|+G\frac{|\w|^3}{k^2}+\epsilon\right)
\end{align*}
\item With the setting $p=q=1/3$, our regret bound takes the form:
\begin{align*}
R_T(\w)&\le \tilde O\left(|\w|G\sqrt{T} + G|\w| + \left(\frac{|\w|^3}{k^2}+k\right)GT^{1/3}\right)
\end{align*}
\end{enumerate}
We note that in all cases it appears that the optimal value of $k$ is $O(|\w|)$, so that $k$ is playing a similar role to the scaling of a learning rate in gradient-descent style algorithms. However, the optimal $k$ does not depend on $T$ and so we retain $O(\sqrt{T})$ regret no matter what value is chosen for $k$. The second example above has the interesting property that for large enough $T$, the dominant term is $\tilde O(\|\w\|G\sqrt{T})$ for any fixed $\w\ne 0$ for \emph{any} choice of $k$ (we remove the $kG\sqrt{T}$ term), so that for large $T$ we obtain the optimal scaling with respect to $|\w|$ even for very small $|\w|$.

\section{Conclusion and Open Problems}

We have presented a new online convex optimization algorithm, \algname, which adapts to both unknown $\|\w\|$ and $G$ while guaranteeing sublinear regret. Although the only $T$-dependent term in \algname's regret bound matches the optimal bound of $\tilde O\left(\|\w\|\sqrt{\sum_{t=1}^T \|g_t\|_\star^2}\right)$, we avoid exponential lower bounds by adding a $T$-independent penalty $O(\|\w\|^3)$. Our algorithm's principle hyperparameter is the value $k$, which ``morally'' should be an estimate of $\|\w\|$. As a result, in the large-$T$ limit, our algorithm obtains regret that grows as $\tilde O(\|\w\|G\sqrt{T})$ without knowledge of either $\|\w\|$ or $G$.

There are at least two natural open problems suggested by this work. First, our technique provides a simple way to ``sidestep'' the lower-bound frontier of \citep{cutkosky2017online}, and so naturally suggests the question of whether there is an extension to this frontier that provides some guidance into whether our regret bounds are optimal. Second, our regret bound maintains a dependence on $\sqrt{\sum_{t=1}^T \|g_t\|_\star}$ rather than $\sqrt{\sum_{t=1}^T \|g_t\|_\star^2}$. The latter bound would provide much better behavior on smooth losses \citep{srebro2010smoothness}, and so we hope future work will yield such an improved algorithm.

\bibliographystyle{plainnat}
\bibliography{all}

\appendix

\section{Dimension-Free Bound in Banach Spaces}\label{sec:dimfree}
In this section, we observe that by use of the one-dimensional to dimension-free reduction proposed by \citep{cutkosky2018black}, we may seamlessly convert the result of Theorem \ref{thm:final} into a dimension-free regret bound, resulting in Algorithm \ref{alg:finaldim}. We give pseudo-code for this reduction for completeness below. 

\begin{algorithm}[h]
\caption{Dimension-Free \algname}
\label{alg:finaldim}
\begin{algorithmic}[1]
{
    \Require{Parameters $k,p,\mathfrak{g},\epsilon,\tau$, Banach space $W$.}
    \State{{\bfseries Initialize: } Instantiate \algname\ with Algorithm \ref{alg:ons_1d}, $k$, $p$ and $\mathfrak{g}$ as $\onedol$. Instantiate an adaptive unit-ball algorithm $\bol$}
    \For{$t=1$ {\bfseries to} $T$}
    \State{Get $x_t$ from $\onedol$}
    \State{Get $y_t$ from $\bol$}
    \State Play $w_t=x_ty_t$.
    \State Receive gradient $g_t$.
    \State Send $g_t$ to $\bol$.
    \State Send $\langle g_t, y_t\rangle$ to $\onedol$.
    \EndFor
}
\end{algorithmic}
\end{algorithm}

\begin{corollary}\label{thm:detailedregretdim}
Suppose $\bol$ guarantees regret $R^{\bol}_T(z)$ for any $z$ in the unit ball.
Then Dimension-Free \algname\ guarantees regret:
\begin{align*}
R_T(\w)&\le 2\epsilon + 2\|\w\| 
\max\left[8h_T \ln \left(\frac{16\|\w\| h_T\exp(\alpha/4h_T^2) \left(1+\frac{\sum_{t=1}^T \|g_t\|^2}{\alpha}\right)^{4.5}}{\epsilon}\right)-h_T,\right.\\
&\quad\quad\left.
 2\sqrt{\sum_{t=1}^T \|g_t\|_\star^2 \ln\left(\frac{4 \left(\sum_{t=1}^T \|g_t\|_\star^2\right)^{10}\exp(\alpha/4h_T^2) \w^2}{\epsilon^2}+1\right)}\right]\\
&\quad\quad + G\left[k\max_{t\le T}\left(\sum_{i=1}^t \|g_i\|_\star/G_t\right)^p + 2|\w|\right] + \min_{q\in [0,1]}\left[G \frac{\|\w\|^{1+\frac{1-q}{p}}}{k^{\frac{1-q}{p}}}\left(\sum_{t=1}^{T} \|g_t\|/G\right)^{q}\right]\\
&\quad\quad+R^{\bol}_T(\w/\|\w\|)
\end{align*}

where $G=\max_t \|g_t\|_\star$ and $h_T=\max(\mathfrak{g}, G)$.
\end{corollary}

As an important special case, when $W$ is a Hilbert space we can obtain $R^{\bol}_T(z)\le 2^{3/2}\sqrt{\sum_{t=1}^T \|g_t\|^2}$ via standard Adagrad-style analysis (which we reproduce in Section \ref{sec:adagrad} for completeness).

\section{Proof of Theorem \ref{thm:onshints}}
We restate Theorem \ref{thm:onshints} below for reference:
\onshints*
\begin{proof}
First, we recall the connection between wealth and regret. If we can prove $\wealth_T\ge f\left(-\sum_{t=1}^T g_t\right)$ for some function $f$, then we have:
\begin{align*}
R_T(\w) \le \epsilon -\w\sum_{t=1}^T g_t -f\left(-\sum_{t=1}^T g_t\right)\le \sup_{X} \epsilon + X\w-f(X)=\epsilon + f^\star(\w)
\end{align*}
where $f^\star$ is the Fenchel conjugate of $f$. Thus it suffices to prove a lower-bound on the wealth of our algorithm.

Define $\wealth(\v)$ as the wealth of an algorithm that uses betting fraction $\v$ on every round. Then we have the recursions:
\begin{align*}
\log(\wealth_T)&=\log(\epsilon) +\sum_{t=1}^T \log(1-v_tg_t)\\
\log(\wealth(\v))&=\log(\epsilon) +\sum_{t=1}^T \log(1-\v g_t)\\
\end{align*}
Now suppose we choose $v_t$ via an online learning algorithm on the losses $-\log(1-vg_t)$, obtaining regret $R_T^v(\v)$. Subtracting the log-wealth equations and exponentiating, we have
\begin{align*}
\wealth_T &\ge \frac{\wealth(\v)}{\exp(R_T^v(\v))}
\end{align*}
Choose $\v = \frac{\sum_{t=1}^T g_t}{2\sum_{t=1}^T g_t^2 + 2h_T|\sum_{t=1}^T g_t|}\in[-1/2h_T, 1/2h_T]$. Then, using $\log(1+x)\ge x-x^2$ for $|x|\le 1/2$, we have
\begin{align*}
\log(\wealth(\v))&\ge \log(\epsilon) +\frac{|\sum_{t=1}^T g_t|^2}{4\sum_{t=1}^T g_t^2 + 4h_T|\sum_{t=1}^T g_t|}
\end{align*}
which implies
\begin{align*}
\wealth_T &\ge \epsilon\frac{\exp\left(\frac{|\sum_{t=1}^T g_t|^2}{4\sum_{t=1}^T g_t^2 + 4h_T|\sum_{t=1}^T g_t|}\right)}{\exp(R_T^v(\v))}
\end{align*}

Now it remains to compute $R_T^v(\v)$. In the standard ONS bound, this is $O(\log(T))$. However, our setting is slightly more subtle because we have the shrinking domains $S_t=[-1/2h_t,1/2h_t]$. It turns out that this has essentially zero effect on the analysis, but we recapticulate the argument in Section \ref{sec:ons} for completeness (see Lemma \ref{thm:log_bound}). The final result is that
\begin{align*}
R^v_T(\v)&\le \frac{\alpha}{4h_T^2} + 4.5\log\left(\frac{\alpha+\sum_{t=1}^T g_t^2}{\alpha}\right)
\end{align*}
from which we obtain
\begin{align*}
\wealth_T &\ge \epsilon\frac{\exp\left(\frac{|\sum_{t=1}^T g_t|^2}{4\sum_{t=1}^T g_t^2 + 4h_T|\sum_{t=1}^T g_t|}\right)}{\exp(\alpha/4h_T^2)\left(\frac{\alpha+\sum_{t=1}^T g_t^2}{\alpha}\right)^{4.5}}
\end{align*}

Set $a=\frac{\epsilon}{\exp(1/4h_T^2)\left(\frac{\alpha+\sum_{t=1}^T g_t^2}{\alpha}\right)^{4.5}}$
, $b= \frac{\alpha}{4h_T}$ and $c=\frac{\sum_{t=1}^T g_t^2}{h_T}$. Then we can write
\begin{align*}
\wealth_T &\ge a\exp\left(b\frac{(\sum_{t=1}^T g_t)^2}{|\sum_{t=1}^T g_t|+c}\right)
\end{align*}
so if we define $f=a\exp\left(b\frac{x^2}{|x|+c}\right)$, we have 
\begin{align*}
R_T(\w)&\le \epsilon + f^\star(\w)
\end{align*}
We recall the computation of $f^\star$ in Lemma \ref{lemma:fenchel_exp2}, to obtain:
\begin{align*}
R_T(\w) &\le \epsilon + |\w| 
\max\left(\frac{2}{b} \left(\ln \frac{2|\w|}{a b}-1\right), \sqrt{\frac{c}{b} \ln\left(\frac{c \w^2}{a^2 b}+1\right)}-a\right)\\
&\le \epsilon + |\w| 
\max\left(8h_T \left(\ln \frac{16|\w| h_T\exp(\alpha/4h_T^2) \left(\frac{\alpha+\sum_{t=1}^T g_t^2}{\alpha}\right)^{4.5}}{\epsilon}-1\right),\right.\\
&\quad\quad\left.
 2\sqrt{\sum_{t=1}^T g_t^2 \ln\left(\frac{4 \left(\sum_{t=1}^T g_t^2\right)^{10}\exp(\alpha/2G_T^2) \w^2}{\epsilon^2}+1\right)}\right)
\end{align*}

\end{proof}

\section{1D ONS with shrinking domains}\label{sec:ons}
Essentially all of the analysis here is identical to the classical procedure (e.g. see \citep{hazan2007logarithmic}), but we recall it here to verify that the shrinking domains have little effect.

\begin{algorithm}[t]
\caption{ONS with shrinking domains}
\label{algorithm:onsdomains}
\begin{algorithmic}[1]
{
    \Require{$\tau,\beta>0$}
    \State{{\bfseries Initialize: } Interval $S_1\subset \R$, $v_1 = 0\in S_1$}
    \For{$t=1$ {\bfseries to} $T$}
    \State{Play $v_t$}
    \State{Receive $z_t$}
    \State{Receive interval $S_{t+1}\subset S_t$}
    \State{Set $A_{t} = \tau+ \sum_{i=1}^t z_i^2$}
    \State{$v_{t+1} = \Pi_{S_{t+1}} \left(v_{t} - \frac{z_t}{\beta A_t}\right)$, where $\Pi_{S_{t+1}}(x)$ is the projection of $x$ to $S_{t+1}$ (i.e. a truncation).}
    \EndFor
}
\end{algorithmic}
\end{algorithm}

\begin{theorem}\label{thm:onsfirstpass}
For any $\v\in S_T$,
\[
\sum_{t=1}^T \left(z_t(v_t - \v)-\frac{\beta}{2}[z_t(v_t -\v)]^2\right)
\leq \frac{\beta \tau}{2}\v^2 + \frac{2}{\beta} \sum_{t=1}^T \frac{z_t^2}{A_t}~.
\]
\end{theorem}
\begin{proof}
Define $x_{t+1}=v_t - \frac{z_t}{\beta A_t}$ so that $v_{t+1}=\Pi_{S_{t+1}}(x_{t+1})$ for $t<T$. We make the definition $v_{T+1}=x_{T+1}$ for ease of analysis later.
Then, we have
\begin{align*}
x_{t+1} - \v = v_t - \v - \frac{z_t}{\beta A_t},
\end{align*}
that implies
\begin{align*}
A_{t}(x_{t+1} - \v) = A_{t}(v_t - \v - \frac{z_t}{\beta A_t}) = A_{t}(v_t - \v) - \frac{1}{\beta} z_t,
\end{align*}
and
\begin{align*}
A_{t}(x_{t+1} - \v)^2&= (A_{t}(v_t - \v) - \frac{1}{\beta} z_t)(x_{t+1} - \v) \\
&= A_{t}(v_t - \v)(x_{t+1} - \v) - \frac{1}{\beta}z_t(x_{t+1} - \v) \\
&= A_{t}(v_t - \v)(x_{t+1} - \v) - \frac{1}{\beta}  z_t(v_t - \v - \frac{z_t}{\beta A_t} ) \\
&= A_{t}(v_t - \v)(x_{t+1} - \v) - \frac{1}{\beta} z_t(v_t - \v)  + \frac{z_t^2 }{\beta^2 A_t} \\
&= A_{t}(v_t - \v)(v_t - \v - \frac{z_t }{\beta A_t} ) - \frac{1}{\beta}z_t(v_t - \v) + \frac{z_t^2}{\beta^2 A_t} \\
&= A_{t}(v_t - \v)^2 - \frac{2}{\beta} (v_t - \v)z_t + \frac{z_t^2}{\beta^2 A_t}
\end{align*}
We now use the definition of $\Pi_{S_{t+1}}$, and the assumption that $\v \in S_T\subset S_{t+1}$ to have:
\[
(x_{t+1} - \v)^2 \geq (v_{t+1} - \v)^2
\]
from which we conclude:
\begin{align*}
z_t(v_t - \v)
&\leq \frac{\beta A_t}{2} (v_t - \v)^2 - \frac{\beta}{2}A_{t}(v_{t+1} - \v)^2   + \frac{2z_t^2 }{\beta A_t}
\end{align*}
Summing over $t=1,\cdots,T$, we have
\begin{align*}
\sum_{t=1}^T z_t(v_t - \v)
&\leq \frac{\beta}{2} A_{1}(v_1 - \v)^2 + \frac{\beta}{2}\sum_{t=2}^T (A_{t}-A_{t-1})(v_t - \v)^2 \\
&\quad - \frac{\beta (v_{T+1} - \v)^2}{2 A_T} + \sum_{t=1}^T \frac{2 z_t^2}{\beta A_t} \\
&\leq \frac{\beta}{2} A_{1}(v_1 - \v)^2 +\frac{\beta}{2}\sum_{t=2}^T z_t^2(v_t-v)^2 + \sum_{t=1}^T \frac{2 z_t^2 }{\beta A_t} \\
&= \frac{\beta}{2} \tau \v^2 +\frac{\beta}{2}\sum_{t=1}^T [z_t(v_t-\v)]^2 + \sum_{t=1}^T \frac{2z_t^2}{\beta A_t} 
\end{align*}
\end{proof}

Next we need to bound the sum $ \sum_{t=1}^T \frac{2z_t^2}{\beta A_t}$, which is easy thanks to the concavity of $\log$:
\begin{lemma}\label{thm:logsum}
\[
\sum_{t=1}^T \frac{z_t^2}{ A_t}\le \log\left(1+\frac{\sum_{t=1}^T z_t^2}{\tau}\right)
\]
\end{lemma}
\begin{proof}
Since $\log(x)$ is concave and $\frac{d}{dx}\log(x)=\frac{1}{x}$, we have $\log(a+b)- \log(a)\ge \frac{b}{a+b}$. Therefore for any $K$ we have
\begin{align*}
\log\left(\frac{\tau+\sum_{t=1}^{K+1} z_t^2}{\tau}\right)-\log\left(\frac{\tau+\sum_{t=1}^K z_t^2}{\tau}\right)&\ge \frac{z_{K+1}^2}{\tau+\sum_{t=1}^{K+1} z_t^2}\\
&=\frac{z_{K+1}^2}{A_{K+1}}
\end{align*}
Summing this identity over all $K<T$ proves the result.
\end{proof}

Here are three lemmas copied over (with occasional mild modification) from \citep{cutkosky2018black}:

\begin{lemma}\label{thm:logbound}
\label{eq:upper_bound_log}
For $-1< x\leq2$, we have
\[
\ln(1+x)\leq x-\frac{2-\ln(3)}{4}x^2~.
\]
\end{lemma}

\begin{lemma}\label{thm:beta}
Define $\ell_t(v)=-\ln(1 - g_tv )$.
Let $|\v|, |v| \leq \frac{1}{2G_t}$ and $|g_t|  \leq G_t$. Then
\begin{align*}
\ell_t(v)-\ell_t(\v)
\leq \ell'_t(v) (v-\v ) -\frac{2-\ln(3)}{2}\frac{1}{2}[ \ell'_t(v)(v-\v)]^2~.
\end{align*}
\end{lemma}
\begin{proof}
We have
\[
\ln(1- g_t\v) 
= \ln(1-g_t v +  g_t(v-\v) )
= \ln(1- g_tv) + \ln\left(1 + \frac{g_t(v-\v)}{1-g_tv}\right)~.
\]
Now, observe that since $1-g_t\v \ge 0$ and $1-g_tv \ge 0$, $1 + \frac{ g_t(v-\v)}{1-g_tv}\ge 0$ as well so that $\frac{ g_t(v-\v)}{1-g_tve}\ge -1$. Further, since $|\v-v|\le \frac{1}{G_t}$ and $1-g_tv \ge 1/2$, $\frac{g_t(v-\v)}{1-g_t v }\le 2$. Therefore, by Lemma \ref{thm:logbound} we have
\[
\ln(1-g_t\v)
\le \ln(1-g_t v) + \frac{ g_t(v-\v)}{1-g_t v} - \frac{2-\ln(3)}{4} \frac{[g_t(v-\v)]^2}{(1- g_tv)^2}~.
\]
Using the fact that $\ell'_t(v)=\frac{g_t}{1- g_tv}$ finishes the proof.
\end{proof}

\begin{lemma}\label{thm:log_bound}
Define $\ell_t(v):[-1/2h_t,1/2h_t]\rightarrow \R$ as $\ell_t(v)=-\ln(1- g_t v)$, where $|g_t|\le h_t$.
If we run ONS in Algorithm~\ref{algorithm:onsdomains} with $\beta=\frac{2-\ln(3)}{2}$, $\tau=4\alpha$, and $S_t=[-1/2h_t, 1/2h_t]$, then for all $\v\in S_T$,
\[
\sum_{t=1}^T \ell_t(v_t)-\ell_t(\v)
\leq \frac{\alpha}{4h_T^2} + 4.5\log\left(\frac{\alpha+\sum_{t=1}^T g_t^2}{\alpha}\right)
\]
\end{lemma}
\begin{proof}
By Lemma \ref{thm:beta}, we have:
\[
\sum_{t=1}^T \ell_t(v_t)-\ell_t(\v)\le \sum_{t=1}^T \ell'_t(v_t) (v-\v ) -\frac{2-\ln(3)}{2}\frac{1}{2}[ \ell'_t(v_t)(v_t-\v)]^2
\]
Then set $z_t = \ell'_t(v_t)$ and use Theorem \ref{thm:onsfirstpass} to obtain
\[
\sum_{t=1}^T \ell_t(v_t)-\ell_t(\v)
\leq 2\beta\alpha \v^2 + \frac{2}{\beta} \sum_{t=1}^T \frac{z_t^2}{A_t}~.
\]
Next, apply Lemma \ref{thm:logsum}:
\[
\sum_{t=1}^T \ell_t(v_t)-\ell_t(\v)
\leq \frac{\beta\alpha}{2h_T^2} + \frac{2}{\beta}\log\left(1+\frac{\sum_{t=1}^T z_t^2}{4\alpha}\right)
\]
Now we observe that $|z_t|=\left|\frac{g_t}{1-g_tv_t}\right|\le 2|g_t|$ so that $\sum_{t=1}^T z_t^2 \le 4\sum_{t=1}^t g_t^2$, yielding
\[
\sum_{t=1}^T \ell_t(v_t)-\ell_t(\v)
\leq \frac{\beta\alpha}{2h_T^2} + \frac{2}{\beta}\log\left(\frac{\alpha+\sum_{t=1}^T g_t^2}{\alpha}\right)
\]
Finally, numerically evaluate $\beta$ to conclude the bound.
\end{proof}

\begin{lemma}[Lemma 19 of \citep{cutkosky2018black}]
\label{lemma:fenchel_exp2}
Let $f(x)=a \exp(b \frac{x^2}{|x|+c})$, where $a,b>0$ and $c\geq0$. Then 
\[
f^\star (\theta) 
\leq |\theta| 
\max\left(\frac{2}{b} \left(\ln \frac{2|\theta|}{a b}-1\right), \sqrt{\frac{c}{b} \ln\left(\frac{c \theta^2}{a^2 b}+1\right)}-a\right)~.
\]
\end{lemma}

\section{Adaptive Unit-Ball Algorithm in Hilbert Spaces}\label{sec:adagrad}

Here we briefly recall some classic analysis of adaptive mirror descent algorithms. For simplicity, we only consider the Hilbert space case, rather than a more general smooth Banach space. More details and more generality can be found in \citep{duchi10adagrad, mcmahan2014survey, hazan2008adaptive}.

\begin{algorithm}[t]
\caption{Adaptive Gradient Descent}
\label{alg:adagrad}
\begin{algorithmic}[1]
{
    \State{{\bfseries Initialize: } Unit ball $B$ in some Hilbert space, $w_1=0\in B$, $\lambda=\sqrt{2}$}
    \For{$t=1$ {\bfseries to} $T$}
    \State{Play $w_t$}
    \State{Receive $g_t$}
    \State{Set $\eta_t = \frac{\lambda}{\sqrt{\sum_{t=1}^T \|g_t\|^2}}$}
    \State{Set $w_{t+1} = \prod_{B}(w_t - \eta_tg_t)$ // $\prod_B(x) = \argmin_{y\in B}\|y-x\|$.}
    \EndFor
}
\end{algorithmic}
\end{algorithm}

\begin{theorem}\label{thm:adagrad}
Algorithm \ref{alg:adagrad} guarantees
\begin{align*}
R_T(\w)\le 2^{3/2}\sqrt{\sum_{t=1}^T \|g_t\|^2}
\end{align*}
for all $\w\in B$.
\end{theorem}
\begin{proof}
\begin{align*}
\|w_{t+1}-\w\|^2 &\le \|w_t - \eta_t g_t -\w\|^2\\
&=\|w_t-\w\|^2 + 2\eta_t \langle g_t, w_t-\w\rangle + \eta_t^2 \|g_t\|^2\\
\langle g_t, w_t-\w\rangle &\le \frac{\|w_t-\w\|^2 - \|w_{t+1}-\w\|^2}{2\eta_t} + \frac{\eta_t}{2}\|g_t\|^2\\
R_T(\w) & \le \frac{\|w_1-\w\|^2}{2\eta_1} - \frac{\|\w_{T+1} -\w\|^2}{\eta_T} +\sum_{t=2}^T \frac{\|w_t-\w\|^2}{2}(\eta_t^{-1} - \eta_{t-1}^{-1}) + \sum_{t=1}^T \frac{\eta_t}{2}\|g_t\|^2\\
&\le \frac{2}{\eta_T} + \lambda \sqrt{\sum_{t=1}^T \|g_t\|^2}\\
&\le \left(\frac{2}{\lambda} + \lambda\right)\sqrt{\sum_{t=1}^T \|g_t\|^2}\\
&= 2^{3/2}\sqrt{\sum_{t=1}^T \|g_t\|^2}
\end{align*}
where we have used the identity $\sum_{t=1}^T \frac{\|g_t\|^2}{\sqrt{\sum_{i=1}^T \|g_i\|^2}}\le 2\sqrt{\sum_{t=1}^T \|g_t\|^2}$, which holds by concavity of the square root.
\end{proof}

\end{document}